\documentclass[letterpaper,10pt,conference]{ieeeconf}
\IEEEoverridecommandlockouts
\makeatletter
\let\proof\@undefined
\let\endproof\@undefined
\makeatother
\IEEEoverridecommandlockouts
\usepackage[pdftex]{graphicx}
  \graphicspath{{./}{./figures/}{./plots/}}
  \DeclareGraphicsExtensions{.pdf,.jpeg,.png}
\usepackage{amsthm, amsmath, amssymb}
\usepackage[ruled,vlined,linesnumbered]{algorithm2e}
  \SetEndCharOfAlgoLine{}

\usepackage{comment,xspace}
\usepackage{amsmath,amssymb}
\usepackage{url}
\usepackage{color}
\usepackage{listings}
\usepackage{xspace}
\usepackage{fixltx2e}
\usepackage{caption}
\usepackage{booktabs}
\usepackage{ulem}
\renewcommand{\paragraph}[1]{\emph{#1:} }

\newtheorem{theorem}{Theorem}[section]
\newtheorem{proposition}[theorem]{Proposition}
\newtheorem{corollary}[theorem]{Corollary}
\newtheorem{lemma}[theorem]{Lemma}

\theoremstyle{definition}
\newtheorem{definition}[theorem]{Definition}
\newtheorem{problem}[theorem]{Problem}

\theoremstyle{remark}
\newtheorem{remark}[theorem]{Remark}

\newcommand{\tuple}[1]{\ensuremath{\left\langle #1 \right\rangle}\xspace}

\newcommand{\RR}{\ensuremath{\mathbb{R}}\xspace}
\newcommand{\ZZ}{\ensuremath{\mathbb{Z}}\xspace}

\newcommand\oprocendsymbol{\hbox{$\square$}}
\newcommand\oprocend{\relax\ifmmode\else\unskip\hfill\fi\oprocendsymbol}

\newcommand{\cnph}{\ensuremath{\text{\textbf{NP}-hard}}\xspace}

\newcommand{\mst}{\ensuremath{\text{M{\scriptsize ST}}}\xspace}
\newcommand{\tsp}{\ensuremath{\text{T{\scriptsize SP}}}\xspace}
\newcommand{\ctsp}{\ensuremath{\text{CT{\scriptsize SP}}}\xspace}

\newcommand{\nestedCtsp}{\ensuremath{\text{CT{\scriptsize SP}}^*}\xspace}
\newcommand{\gtsp}{\ensuremath{\text{G{\scriptsize TSP}}}\xspace}
\newcommand{\tsplib}  {\textsc{TspLib}\xspace}

\newcommand{\gammaCluster}{\ensuremath{\Gamma}-Cluster\xspace}
\newcommand{\gammaClustering}{\ensuremath{\Gamma}-Clustering\xspace}

\urldef{\smith}\url{stephen.smith@uwaterloo.ca}
\urldef{\frank}\url{fcimeson@uwaterloo.ca}

\title{
  \Large \bf Gamma Clustering for Path Planning:
  A New, Efficient Clustering Approach Geared Towards Preserving Minimal Length Solution Paths
  }
  
  \title{\Large \bf Clustering in Discrete Path Planning for Approximating Minimum Length Paths}
\author{
  Frank Imeson \qquad Stephen L. Smith
  \thanks{ This research is partially supported by the Natural Sciences and Engineering Research Council of Canada (NSERC). }
  \thanks{ The authors are with the Department of Electrical and Computer Engineering, University of Waterloo, Waterloo ON, N2L 3G1 Canada  (\frank; \smith) }
}
\date{}

\begin{document}
\maketitle

%******************************************************************************************
% 
%******************************************************************************************
\begin{abstract}
In this paper we consider discrete robot path planning problems on metric graphs.  We propose a clustering method, \gammaClustering for the planning graph that significantly reduces the number of feasible solutions, yet retains a solution within a constant factor of the optimal.  By increasing the input parameter $\Gamma$, the constant factor can be decreased, but with less reduction in the search space.  We provide a simple polynomial-time algorithm for finding optimal \gammaCluster{s}, and show that for a given $\Gamma$, this optimal is unique. We demonstrate the effectiveness of the clustering method on traveling salesman instances, showing that for many instances we obtain significant reductions in computation time with little to no reduction in solution quality.
\end{abstract}

%******************************************************************************************
% 
%******************************************************************************************
\section{Introduction}
\label{sec:intro}
Discrete path planning is at the root of many robotic applications, from surveillance and monitoring for security, to pickup and delivery problems in automated warehouses. In such problems the environment is described as a graph, and the goal is to find a path in the graph that completes the task and minimizes the cost function.  For example, in monitoring, a common problem is to compute a tour of the graph that visits all vertices and has minimum length~\cite{smith2011optimal}. These discrete planning problems are typically \cnph~\cite{smith2011optimal, gouveia2015load}, and thus there is a fundamental trade off between solution quality and computation time. In this paper we propose a graph clustering method, called \gammaClustering, that can be used to reduce the space of feasible solutions considered during the optimization.  The parameter $\Gamma$ serves to trade-off the feasible solution space reduction (and typically computation time) with the quality of the resulting solution.

The idea behind \gammaClustering is to group vertices together that are in close proximity to each other but are also far from all other vertices. Figure~\ref{fig:example} shows an example of \gammaClustering in an office environment. Given this clustering, we solve the path planning problem by restricting the path to visit vertices within each cluster consecutively (i.e., no path can visit any cluster more than once). This restriction reduces the number of possible solutions exponentially and thus reduces the amount of computational time needed to find good solutions.

Unlike other clustering methods, \gammaClustering does not accept as input a desired number of clusters.  This means that some instances will have no clusters, while others will have many. In this way, \gammaCluster{s} only explore natural structures within the problem instances instead of imposing structures onto the instance.   Additionally, when the graph is metric, we establish that for a given \gammaClustering, the optimal path of the clustered planning problem is within a constant factor (dependent on $\Gamma$) of the true optimal solution.

\begin{figure}
  \centering
  \includegraphics[height=149pt, width=200pt]{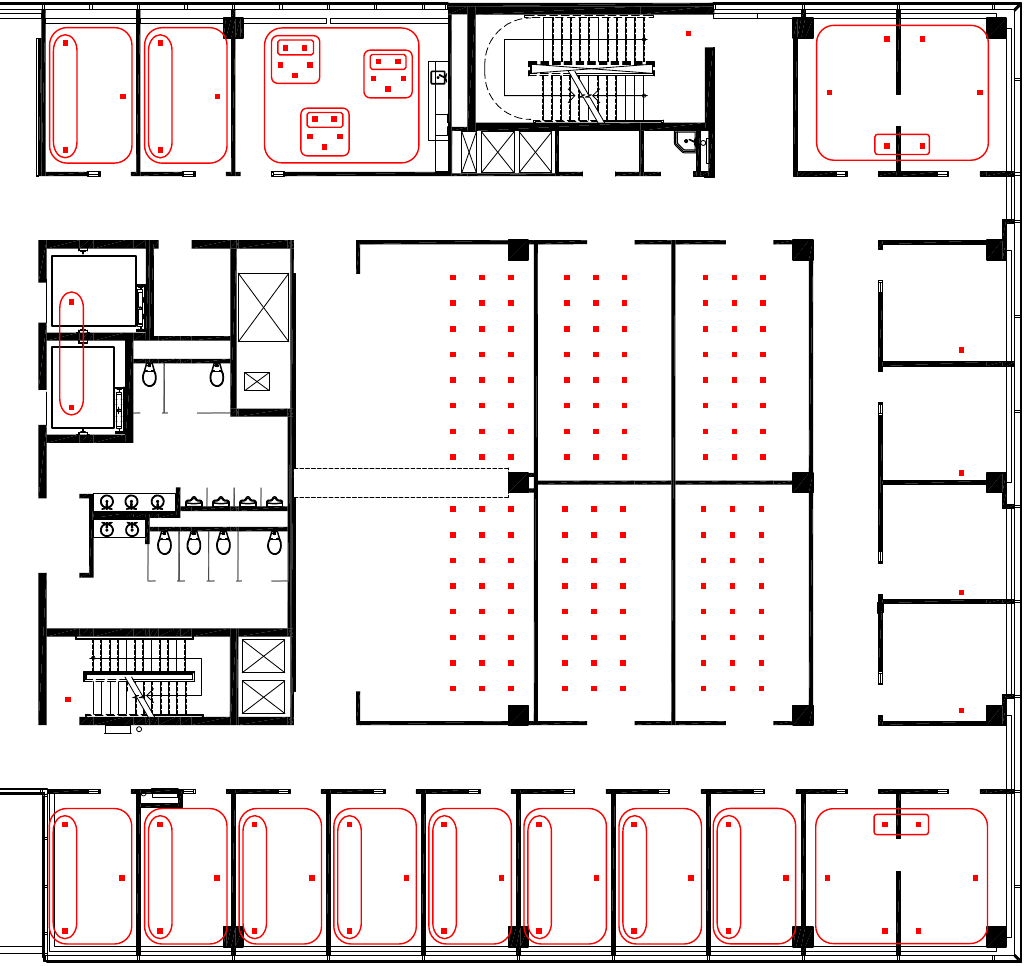}
  \caption{
    The results of \gammaClustering used on an office environment. The red dots
    represent locations of interest and the red rectangles show clusters of size
    two or greater.
  }
  \label{fig:example}
\end{figure}

\paragraph{Related work} 
There are a number of clustering methods for Euclidean~\cite{jain1988algorithms, karypis1999chameleon} and discrete~\cite{kernighan1970efficient, guha2000rock, katselis2016clustering} environments. Typically the objective of these algorithms is to find a set of equal (or roughly equal) non-overlapping clusters that are grouped by similarity (close in proximity, little to no outgoing edges, etc \ldots), where each location in the graph is assigned to one cluster. For these methods, the desired number of clusters is given as an input parameter. In contrast, in \gammaClustering, the idea is to simply find a specific form of clustering within the environment, if it exists. These \gammaCluster{s} may be nested within one another.

There are other clustering methods that also look for specific structures within the graph such as community structures~\cite{blondel2008fast}, which is based on a metric that captures the density of links within communities to that between communities. In contrast, our clustering method is specifically designed to find structures that yield desirable properties for path planning on road maps.

The use of clustering to save on processing/time is done in a variety of different fields such as data mining~\cite{berkhin2006survey}, parallel computer processing~\cite{meyerhenke2009new}, image processing~\cite{gao2010kernel}, and control~\cite{jin2012multi} for path planning. In environments that have regions with a high degree of connectivity, such as electronic circuits, clustering is commonly used to identify these regions and then plan (nearly) independently in each cluster~\cite{karypis1998fast}. For path planning problems with repetitive tasks, one can cluster a set of popular robot action sequences into macros~\cite{backstrom2014automaton, levihn2013foresight}, allowing the solver to quickly discover solutions that benefit from these action sequences. In applications such as sensor sweeping for coverage problems~\cite{das2014mapping} or in the routing of multiple agents~\cite{ryan2008exploiting}, clustering has been used to partition the environment into regions that can again be treated in a nearly independent manner, reducing computation time.

There is some prior work on partitioning in discrete path planning. Multilevel refinement~\cite{chevalier2009comparison} is the process of recursively coarsening the graph by aggregating the vertices together to create smaller and smaller instances, for which a plan can be found more easily. The plan is then recursively refined to obtain a solution to the original problem. The idea in coarsening a graph is that the new coarse edges should approximate the transition costs in the original graph.  This differs from \gammaClustering, which preserves the edges within the graph. There are a number of clustering approaches that aim to reduced the complexity of Euclidean and or planar \tsp problems~\cite{karp1977probabilistic, haxhimusa2009approximative}. \gammaClustering is more general, in that it works on any graph, while the solution quality guarantees only hold for metric graphs.

\paragraph{Contributions} The main contribution of this paper is the introduction of \gammaClustering, a clustering method for a class of discrete path planning problems.  We establish that the solution to the corresponding clustered problem provides a $\min\left(2, 1 + \frac{3}{2\Gamma}\right)$-factor approximation to the optimal solution. We give some insight into the reduction of the search space as a function of the amount of clustering, and we provide an efficient algorithm for computing the optimal \gammaClustering. We then use an integer programming formulation of the \tsp problem to demonstrate that for many problem instances the clustering method reduces the computation time while still finding near-optimal solutions.

%******************************************************************************************
% 
%******************************************************************************************
\section{Preliminaries in Discrete Path Planning}
\label{sec:background}
In this section we define the class of problems considered in this paper, review some semantics of clusters, review the traveling salesman problem (\tsp)~\cite{applegate2006traveling} and define its clustered variant, the general clustered traveling salesman problem (\nestedCtsp).

%******************************
%
%******************************
\subsection{Discrete Path Planning}
Given a graph $G = \tuple{V,E,w}$, we define a \emph{path} as a non-repeating sequence of vertices in $V$, connected by edges in $E$. A \emph{cycle} is a path in which the first and last vertex are equal, and for simplicity we will also refer to cycles as paths. Let $\mathcal{P}$ represent the set of all possible paths in $G$. Then, abstractly, a path planning constraint defines a subset $\mathcal{P}_1 \subseteq \mathcal{P}$ of feasible paths. Given a set of constraints $\mathcal{P}_1, \mathcal{P}_2, \ldots, \mathcal{P}_m$, the set of all feasible paths is $\mathcal{P}_1 \cap \mathcal{P}_2 \cap \cdots \cap \mathcal{P}_m$. In this paper we restrict our attention to the following class of constraints and planning problems.

\begin{definition}[Order-Free Constraints]
	A constraint $\mathcal{P}_1$ is \emph{order-free} if, given any $p\in
	\mathcal{P}_1$, then all paths obtained by permuting the vertices of $p$ are
	also in $\mathcal{P}_1$.
\end{definition}

\begin{problem}[Discrete Path Planning Problem]
  Given a complete weighted graph $G = \tuple{V,E,w}$ and a set of order-free constraints $\{
  \mathcal{P}_1, \mathcal{P}_2, \ldots, \mathcal{P}_m \}$, find the minimum
  length feasible path.
\end{problem}

Many discrete path planning problems for single and multiple robots fall into this class, so long as they do not restrict the ordering of vertex visits (i.e., no constraints of the form ``visit $A$ before $B$''). Some examples include single and multi-robot traveling salesman problems, point-to-point planning, and patrolling.  As a specific example the \gtsp is a problem where a robot is required to visit exactly one location in each non-overlapping set of locations~\cite{noon1993efficient}. This is naturally expressed in the above framework by having one constraint for each set: for each cluster $V_i$ we have a constraint stating that exactly one vertex in $V_i$ must be visited in the path. 

A \emph{metric} discrete path planning problem is one where the edge weights in the graph $G$ satisfy the triangle inequality: for $v_a,v_b,v_c\in V$, we have $w(v_a,v_c) \leq w(v_a, v_b) + w(v_b, v_c)$.

To describe the number of feasible paths for a given planning problem, we use the phrase \emph{search space size}. For example, a problem where we must choose an ordering of $n$ locations has a search space size of $n!$, since there are $n!$ combinations that a path may take. Note that as more constraints are added to the problem, the search space size can only be reduced, since a feasible path must lie in the intersection of all constraints.

%******************************
%
%******************************
\subsection{Clusters}
A \emph{cluster} is a subset of the graph's vertices, $V_i \subset V$. Given the clusters $V_1$ and $V_2$, we say $V_1$ is \emph{nested} in $V_2$ if $V_1 \subseteq V_2$. The clusters $V_1$ and $V_2$ are \emph{overlapping} if $V_1 \cap V_2 \neq \emptyset$, $V_1 \nsubseteq V_2$, and $V_2 \nsubseteq V_1$. A \emph{set of clusters} (or clustering) is denoted by $C = \{V_1,\ldots,V_m\}$. A clustering $C = \{V_1, V_2, \ldots, V_m \}$ is a nested if there exists some $V_i \subseteq V_j$ for $V_i, V_j \in C$.

In this paper we seek to add clustering constraints to a discrete planning problem that reduce its search space size, but also retains low-cost feasible paths. The clustering constraints we consider are of the following form.

\begin{definition}[Consecutive Visit Constraint]
Given a graph $G  = \tuple{V,E,w}$ and a cluster $V_i \subseteq V$, a feasible path $p$ must visit the vertices within the cluster $V_i$ \emph{consecutively}. Formally, the vertices visited by $p$, $V[p]$ are visited consecutively if there exists a path segment $p'$ of $p$ that visits every vertex in $V_i \cap V[p]$ and is of length $\left|V_i \cap V[p]\right|$.
\end{definition}

Note that in the above definition, it is not necessary for all of the vertices in $V_i$ to be visited. It is just necessary to visit the vertices consecutively in $V_i$ that are visited.

%******************************
%
%******************************
\subsection{Traveling Salesman Problems}
The traveling salesman problem (\tsp) is defined as follows: given a set of cities and distances between each pair of cities, find the shortest path that the salesman can take to visit each city exactly once and return to the first city (i.e., the shortest tour). A tour in a graph that visits each vertex exactly once is called a Hamiltonian cycle (regardless of path cost). The general clustered version of \tsp is the extension that requires the solution to visit the vertices within the clusters consecutively. The definition of these problems is as follows:

\begin{problem}[Traveling Salesman Problem (\tsp)]
Given a complete graph $G = \tuple{V,E,w}$ with edge weights $w: E \rightarrow \RR_{\geq 0}$, find a Hamiltonian cycle of $G$ with minimum cost.  
\end{problem}

\begin{problem}[General-\ctsp]
  Given a complete and weighted graph $G = \tuple{V,E,w}$ along with a
  clustering $C = \{V_1,\ldots, V_m\}$, find a Hamiltonian cycle of $G$ with
  minimum cost such that the vertices within each cluster $V_i$ are visited
  consecutively.
\end{problem}

The traditional version of the \ctsp restricts the clusters to be non-overlapping (and non-nested). For this paper we use the syntax \nestedCtsp to emphasize when we are solving a General-\ctsp problem \ctsp to refer to the traditional problem.

%************************************************************
%
%************************************************************
\section{\gammaClustering}
\label{sec:theory}
In this section, we define \gammaClustering and show that the \gammaCluster{ed} path planning problem provides a $\min\left(2, 1 + \frac{3}{2\Gamma}\right)$ approximation of the original path planning problem. We then describe an algorithm for finding the optimal \gammaClustering, and characterize the search space reductions.

%******************************
%
%******************************
\subsection{Definition of \gammaClustering}
Below we define the notion of \gammaCluster{s}, \gammaClustering{s}, and the clustered discrete path planning problem. Then we pose the clustering problem as one of maximizing the search space reduction.

\begin{definition}[$\Gamma$-Metric of a cluster]
  \label{def:gamma_cluster}
  Given a graph $G = \tuple{V,E,w}$ and a cluster $V_i \subset V$, we define
  the following quantities for $V_i$ relative to $G$:
  \begin{align*}
    \alpha_i &\equiv \min_{v_a \in V_i, v_b \in V \setminus V_i} (w(v_a, v_b), w(v_b, v_a)) \\
    \beta_i &\equiv \max_{v_a,v_b \in V_i, v_a \neq v_b} w(v_a, v_b) \\
    \Gamma_i &\equiv \frac{\alpha_i}{\beta_i},
  \end{align*}
  where $\alpha_i$ represents the minimum weight edge entering or exiting the
  cluster $V_i$, and $\beta_i$ represents the maximum weight edge within $V_i$. The
  ratio $\Gamma_i$ is a measure of how separated the vertices in $V_i$ are from
  the remaining vertices in $G$.
\end{definition}

\begin{definition}[\gammaClustering]
  \label{def:gamma_clusters}
  Given an input parameter $\Gamma \geq 0$ and a graph $G = \tuple{V,E,w}$, a
  clustering $C = \{V_1, V_2,\ldots,V_m\}$ is said to be a \gammaClustering if
  and only if $V$ is covered by $V_1 \cup V_2 \cup \cdots \cup V_m$; each $V_i
  \in C$ has a separation $\Gamma_i \geq \Gamma$; and the clusters are
  either nested ($V_i \subseteq V_j$ or $V_j \subseteq V_i$ for all $V_i, V_j \in C$)
  or non-overlapping ($V_i \cap V_j = \emptyset$).
\end{definition}

The search space reduction for path planning problems comes from restricting paths to visit the clusters consecutively and our goal is to maximize that reduction. Thus we are interested in the following two problems.

\begin{definition}[The Clustered Path Planning Problem]
  \label{def:path_planning_problem}
  Given a discrete path planning problem $P$ and a clustering
  $C$, the \emph{clustered version of the problem} $P'$ has
  the constraint that the path must visit the vertices within each
  cluster consecutively, in addition to all the constraints of $P$.
\end{definition}

\begin{definition}[The Clustering Problem]
  \label{def:clustering_problem}
  Given a graph $G = \tuple{V,E,w}$ and a parameter $\Gamma > 0$, find a
  \gammaClustering $C^*$ such that the search space reduction is maximized.
\end{definition}

\begin{remark}[Overlap]
Note that in Definition~\ref{def:gamma_clusters} overlapping clusters are not permitted. This is necessary for the problem in Definition~\ref{def:path_planning_problem} to be well defined. In addition, we will see in the following section that clusters that have a separation of $\Gamma_i > 1$ cannot overlap.
\end{remark}

%******************************
%
%******************************
\subsection{Solution Quality Bounds}
In this section, we show that when the graph $G$ is metric and $\Gamma > 1$, then the solution to the $\Gamma$-clustered path planning problem provides a $\min\left(2, \left(1 + \frac{3}{2\Gamma}\right)\right)$-factor approximation to the optimal.

%******************************
%
%******************************
\begin{theorem}[Approximation Factor]
  \label{thrm:approximation_factor}
  Given a metric discrete path planning problem $P$ with optimal solution $p^*$
  and cost $c^*$, a \gammaClustering $C = \{V_1, V_2, \ldots, V_m \}$ where
  $\Gamma > 1$, then the optimal solution $(p')^*$ to the clustered problem $P'$
  over the same set of vertices is a solution to $P$ with cost $(c')^* \leq
  \min\left(2, 1 + \frac{3}{2\Gamma}\right) c^*$.
\end{theorem}

To prove the main result, we begin by proving the first half of the bound $(c')^* \leq 2c^*$. We do this by using an minimum spanning tree (MST) to construct a feasible solution for $P'$.

\begin{lemma}[Minimum Spanning Tree (MST)]
  \label{thrm:mst_edges}
  Given a metric graph $G$ and a \gammaClustering $C$, with $\Gamma > 1$, every
  MST will have exactly one inter-set edge for each cluster $V_i \in C$.
\end{lemma}

\begin{proof}
  We prove the above result by contradiction. Suppose the MST has at least two
  inter-set edges connected to $V_i$. Thus, there are at least two sets of
  vertices in $V_i$ that are not connected to each other using intra-set edges.
  We can then lower the cost of the MST by removing one of these inter-set edges
  of weight $\geq \alpha_i$ and replace it with an intra-set edge of weight
  $\leq \beta_i < \alpha_i$. This highlights the contradiction and thus every
  MST will have exactly one inter-set edge for each cluster $V_i \in C$.
\end{proof}

%******************************
%
%******************************
\begin{lemma}[Two-Factor Approximation]
  \label{thrm:2_factor_approximation}
  Consider a metric discrete path planning problem $P$ with an optimal solution
  path $p^*$ and cost $c^*$. Then given a \gammaClustering $C = \{V_1, V_2,
  \ldots, V_m \}$ with $\Gamma > 1$, the optimal solution $(p')^*$ for the
  clustered problem $P'$ over the same set of vertices $V[p^*]$ is a solution to
  $P$ with cost $(c')^* \leq 2 c^*$.
\end{lemma}

\begin{proof}
  To prove the above result, we will use the MST approach described below and in
  \cite{korte2012combinatorial} to construct a path $p'$ over the set
  of vertices in $V[p^*]$. This approach yields a solution $p'$ for $P'$ that
  has our desired cost bound, $c' \leq 2c^*$~\cite{korte2012combinatorial}.
  The MST procedure is described below.
  \begin{enumerate}
    \item Find a minimum spanning tree for the vertices $V[p^*]$.
    \item Duplicate each edge in the tree to create a Eulerian graph.
    \item Find an Eulerian tour of the Eulerian graph.
    \item Convert the tour to a \tsp: if a vertex is visited more than once, after the first visit, create a
          shortcut from the vertex before to the one after, i.e., create a tour
          that visits the vertices in the order they first appeared in the tour.
  \end{enumerate}
  
  What remains is to prove that the above tour is a feasible solution for $P'$. First we
  note that the above approach yields a single tour of all the vertices in
  $V[p^*]$, i.e., there are no disconnected tours. Next we note that
  Lemma~\ref{thrm:mst_edges} states that every MST uses exactly one inter-set
  edge for each cluster $V_i \in C$. Thus when the edges are duplicated and a
  Eulerian tour is found, there are only two inter-set edges used for each $V_i
  \in C$. Furthermore short-cutting the path does not change the number of
  inter-set edges used by the tour, thus the final solution $p'$ only has
  one incoming and one outgoing edge for each cluster $V_i \in C$ and so it is a
  clustered solution for $P'$ that satisfies the bound since the \mst approach
  also yields a solution with cost $c' \leq 2c^*$.
\end{proof}

Next we prove the second half of the bound $(c')^* \leq \left(1 + \frac{3}{2\Gamma}\right)c^*$ by using Algorithm~\ref{alg:deform} to construct a feasible solution $p'$ for $P'$. Additionally we use a modified graph $\hat{G}$ defined in Definition~\ref{def:G_hat} to show that the cost of this solution satisfies our desired bound.

%******************************
%
%******************************
\begin{algorithm}
  \caption{\textsc{deform}$(p,V_i)$}
  \label{alg:deform}
  \If{$p$ has two or less inter-set edges for $V_i$}{
    \Return $p$
  }
  $k \leftarrow 1$ \\
  $p' \leftarrow [\;]$ \tcc*{empty array}
  \While{$p[k] \not\in V_i$}{
    $p'.\textup{append}(p[k])$ \\
    $k \leftarrow k + 1$ \\
  }
  \For{$l \in [k, k+1,\ldots,|p|]$}{
    \If{$p[l] \in V_i$}{
      $p'.\textup{append}(p[l])$
    }
  }
  \For{$l \in [k, k+1,\ldots,|p|]$}{
    \If{$p[l] \not\in V_i$}{
      $p'.\textup{append}(p[l])$
    }
  }
  \Return $p'$
\end{algorithm}

\begin{lemma}[Correctness of Algorithm~\ref{alg:deform}]
  \label{thrm:alg_correctness}
  Given a feasible path $p$ for $P$ and a cluster $V_i \in C$ that is not
  visited consecutively, then $p_i \leftarrow \textsc{deform}(p,V_i)$ does visit
  $V_i$ consecutively and any subsequent deformed paths $p_{i+n+1} \leftarrow
  \textsc{deform}(p_{i+n},V_{i+n+1})$ also visit $V_i$ consecutively, for $n
  \in \ZZ_{> 0}$.
\end{lemma}

\begin{proof}
  By construction $V_i$ is visited consecutively in $p_i \leftarrow
  \textsc{deform}(p,V_i)$. The remaining claim that $V_i$ continues to be
  visited consecutively is proven by showing that the number of inter-set edges
  for $V_i$ in the subsequent paths remains the same.
  
  We prove this result by contradiction. Suppose there is a sequence 
  $\tuple{v_a, v_b} \in p_{j-1}$ for some $j > i$ such that $v_a, v_b \in V_i$
  and $v_a$ is no longer connected to $v_b$ in $p_{j} \leftarrow
  \textsc{deform}(p_{j-1},V_j)$; that is, $\tuple{v_a, v_b} \not\in p_{j}$.
  
  This would mean that one or more vertices were inserted in between $v_a$ and
  $v_b$, thus creating a new inter-set edge. In Algorithm~\ref{alg:deform}, this
  cannot happen in lines 5-7 since the path for the $1^{st}$ vertex to the
  $k^{th}$ is unchanged. This also cannot happen in lines 8-9 since
  this part of the algorithm is only connecting vertices within the cluster
  $V_j$ together. Finally, this cannot happen in lines 11-13, since the
  path is not changing the order of the appearance of $v_a$ and $v_b$ (no
  insertions, just deletions).
  
  Thus there are no additional inter-set edges created by
  Algorithm~\ref{alg:deform} for cluster $V_i$, which highlights our
  contradiction. Therefore all subsequent paths must also visit $V_i$
  consecutively.
\end{proof}

\begin{remark}[Uniqueness]
  \label{thrm:deform_order}
  When Algorithm~\ref{alg:deform} is applied to $p^*$ iteratively for each $V_i
  \in C$ to generate the solution $p'$, the solution is unique despite the order
  that $\textsc{deform}(p,V_i)$ was called for all $V_i \in C$. Furthermore the
  order of the clusters is determined by their first appearance in $p$.
  
  This follows from the method in which Algorithm~\ref{alg:deform} reorders
  the vertices within the tour. Specifically, the order of vertices within each cluster $V_i$ is preserved as $\textsc{deform}(p,V_i)$ is
  called as is the ordering of the remaining vertices.
\end{remark}

%******************************
%
%******************************
\begin{lemma}[Deformation cost in $G$]
  \label{thrm:deform_cost}
  Consider a feasible path $p$ for $P$ that has $2(n+1) \geq 4$ inter-set edges
  for \gammaCluster $V_i$ such that $\Gamma_i > 1$ and $n \in \ZZ_{\geq 0}$.
  Then the cost to deform $p$ into $p' \leftarrow \textsc{deform}(p,V_i)$ is $c'
  - c \leq (2n+1)\beta_i$.
\end{lemma}

\begin{proof}
  In this proof we analyze the cost to deform $p$ into $p'$, which is a result
  of calling $p' \leftarrow \textsc{deform}(p,V_i)$
  (Algorithm~\ref{alg:deform}). There are three types of deformations that
  result from the algorithm: 1) there are short cuts created within the
  cluster; 2) there are short cuts created outside of the cluster; and 3) there
  is a new outgoing edge for the cluster. These deformations are illustrated in
  Figure~\ref{fig:deform03} and a classification of the edges in the figure are
  as follows: 1) edges \tuple{v_3, v_4} and \tuple{v_6, v_7} are short cuts
  within the cluster; 2) edges \tuple{v_{10}, v_{11}} and \tuple{v_{12}, v_{13}}
  are short cuts outside of the cluster; and 3) edge \tuple{v_8, v_9} is the new
  outgoing the edge for the cluster.

  We start by examining the incurred cost to short cut paths within the cluster.
  Consider a path segment \tuple{v_a, v_b, v_c, \ldots, v_x, v_y, v_z} of $p$
  such that $v_a$ is directly connected to $v_z$ in $p'$ with the edge
  \tuple{v_a, v_z} and $v_a, v_z \in V_i$. The incurred cost of each of these
  edges is $\leq \beta_i$ due to the fact that the cost of any intra-set edge
  has weight $\leq \beta_i$. There are $n$ such shortcuts of this nature
  incurred from performing \textsc{deform}$(p,V_i)$ ($n$ captures the
  number of extra visits to the cluster), and so the total incurred cost for
  this type of shortcut is $\leq n\beta_i$.

  Next we examine the incurred cost to short cut paths outside of the cluster.
  Consider a path segment \tuple{v_a, v_b, v_c, \ldots, v_x, v_y, v_z} of $p$
  such that $v_a$ is directly connected to $v_z$ in $p'$ with the edge
  \tuple{v_a, v_z}, $v_a, v_z \not \in V_i$ and $v_b, v_c, \ldots, v_y \in V_i$.
  The incurred cost for each of these short cuts is again $\leq \beta_i$. This
  is due to the metric property of $G$: The cost of the direct path from
  $v_a$ to $v_z$ is less than or equal to any path from $v_a$ to $v_z$,
  specifically $c(v_a, v_z) \leq c(v_a, v_b) + c(v_b, v_y) + c(v_y, v_z) \leq
  c(v_a, v_b) + \beta_i + c(v_y, v_z)$. Thus the incurred cost $\Delta$, of this
  shortcut is bounded by the difference between the cost of the new edges in
  $p'$ and the removed edges in $p^*$, namely:
  \begin{align*}
     \Delta &= c(v_a, v_z) - c(v_a,v_b) - c(v_y, v_z) \\
            &\leq c(v_a, v_b) + \beta_i + c(v_y, v_z) - c(v_a,v_b) - c(v_y, v_z) \\
            & \leq \beta_i 
  \end{align*}
  There are $n$ such shortcuts of this nature incurred by
  \textsc{deform}$(p,V_i)$, and so the total incurred cost for this type of
  shortcut is also $\leq n\beta_i$.

  Lastly we examine the incurred cost of the new outgoing edge. Consider the
  path \tuple{v_a, v_b, v_c, \ldots, v_x, v_y, v_z} of $p$ such that
  \tuple{v_b,v_c} is the first outgoing edge of $V_i$ and \tuple{v_x,v_y} is the
  last outgoing edge of $V_i$, thus \tuple{v_x, v_c} is the new outgoing edge.
  Then due to the metric property we know that $c(v_x, v_c) \leq c(v_x,v_b) +
  c(v_b,v_c) \leq \beta_i + c(v_b,v_c)$. The incurred cost of this deformation
  is the difference between the cost of the new edge \tuple{v_x,v_c} and the
  removed edge \tuple{v_b,v_c} (this edge has not been considered in any
  previous incurred cost calculation):
  \begin{align*}
     \Delta &= c(v_x, v_c) - c(v_b,v_c) \\
            &\leq \beta_i + c(v_b,v_c) - c(v_b,v_c) \\
            & \leq \beta_i 
  \end{align*}

  This accounts for all of the incurred costs, and so the total cost to deform
  $p$ into $p'$ via \textsc{deform}$(p,V_i)$ is $c' - c \leq (2n+1)\beta_i$.
\end{proof}

We introduce the modified graph $\hat{G}$ in the following definition to aid with our ongoing proof of the bound.

%******************************
%
%******************************
\begin{definition}[The modified graph]
  \label{def:G_hat}
  Given a graph $G$ and a \gammaClustering $C$, \emph{the modified graph} $\hat{G}$ is
  a copy of $G$ with the following modifications: if $\tuple{v_a, v_b}$ is an
  inter-set edge with $v_a \in V_i$ and $v_b \in V_j$ then $\hat{c}(v_a, v_b) =
  c(v_a, v_b) + \frac{3}{2}\max(\beta_i, \beta_j)$, otherwise $\hat{c}(v_a, v_b)
  = c(v_a, v_b)$, where $\beta_i$ and $\beta_j$ are as defined in
  Definition~\ref{def:gamma_cluster}.
\end{definition}

\begin{lemma}[Deformation cost in $\hat{G}$]
  \label{thrm:no_deform_cost}
  Consider a feasible path $p$ for $P$ and a \gammaCluster $V_i$ such that
  $\Gamma_i > 1$. Then the cost to deform $p$ into $p' \leftarrow
  \textsc{deform}(p,V_i)$ in $\hat{G}$ is $\hat{c}' - \hat{c} \leq 0$.
\end{lemma}

\begin{proof}
  In this proof we analyze the cost to deform $p$ into $p'$in $\hat{G}$, which
  is a result of single call $p' \leftarrow \textsc{deform}(p,V_i)$
  (Algorithm~\ref{alg:deform}).
  
  From Lemma~\ref{thrm:deform_cost} we see that the cost to deform $p$ into $p'$ with
  respect to $G$ is $c' - c \leq (2n+1)\beta_i$, where there are $2(n+1) \geq 4$
  inter-set edges for $V_i$ in $p$. The cost of $p$ in $\hat{G}$ is
  \[ \hat{c} \geq c + 2(n+1)\left(\frac{3}{2}\right)\beta_i = c + 3(n+1)\beta_i \]
  and the cost of $p'$ in $\hat{G}$ is
  \[ \hat{c}' \leq c' + 2\left(\frac{3}{2}\beta_i\right) \leq c + (2n+1)\beta_i + 3\beta_i. \]
  Thus
  \begin{align*}
    \hat{c}' - \hat{c} 
        &\leq (2n+1)\beta_i + 3\beta_i - 3(n+1)\beta_i \\
        &=    2n\beta_i + \beta_i + 3\beta_i - 3n\beta_i - 3\beta_i \\
        &=    \beta_i - n\beta_i \\
  \end{align*}
  and since $n \geq 1$ (otherwise we did not need to deform the path), then
  $\hat{c}' - \hat{c} \leq 0$.
\end{proof}

\begin{lemma}[Approximation Factor]
  \label{thrm:gamma_factor_approximation}
  Given a metric discrete path planning problem $P$ with optimal solution cost
  $c^*$, $\Gamma > 1$ and a \gammaClustering $C = \{V_1, V_2, \ldots, V_m \}$,
  then the optimal solution $(p')^*$ to the clustered problem $P'$ over the same
  set of vertices is a solution to $P$ with cost $(c')^* \leq (1 +
  \frac{3}{2\Gamma}) c^*$.
\end{lemma}

\begin{proof}
To prove the above result we will work with the modified graph $\hat{G}$ (defined in Definition~\ref{def:G_hat}) and use the result from Lemma~\ref{thrm:no_deform_cost} to show that there exists a clustered solution in $\hat{G}$ that has the same cost or lower than $\hat{c}(p^*)$. Then we will relate $(c')^*$ to $c^*$.

First we show that there exists a clustered solution $p'$ that satisfies the following:
\[ \hat{c}(p') \leq \hat{c}(p^*) \]
To find a solution $p'$ that satisfies the above we will use Algorithm~\ref{alg:deform} to deform $p^*$ into $p'$. The deform algorithm is called for each $V_i \in C$ in any order as $p_{i+1} = \textsc{deform}(p_i,V_i)$ to form a solution for $P'$ (see Lemma~\ref{thrm:alg_correctness}). For each call of $p_{i+1} = \textsc{deform}(p_i,V_i)$ the incurred cost $\hat{c}(p_{i}) - \hat{c}(p_{i+1}) \leq 0$ (see Lemma~\ref{thrm:no_deform_cost}). Thus after the series of calls, we have a clustered solution $p'$ satisfying
\[ \hat{c}(p') \leq \hat{c}(p^*).\]

Next we relate $c(p^*)$ to $\hat{c}(p^*)$ by observing the following for an inter-set edge $\tuple{v_a, v_b} \in p^*$:
\begin{align*}
  \hat{c}(v_a, v_b) &= c(v_a, v_b) + \frac{3}{2}\max(\beta_i, \beta_j) \\
                    &= c(v_a, v_b) + \frac{3}{2}\max\left(\frac{\alpha_i}{\Gamma_i}, \frac{\alpha_j}{\Gamma_j}\right) \\
                    &\leq \left( 1 + \frac{3}{2}\max\left(\frac{1}{\Gamma_i}, \frac{1}{\Gamma_j} \right) \right) c(v_a, v_b) \\
                    &\leq \left( 1 + \frac{3}{2\Gamma} \right) c(v_a, v_b)
\end{align*}
The above inequality for $\hat{c}(p^*)$ is true for each edge $\tuple{v_a, v_b} \in p^*$, inter-set edge or not. Therefore
\[ \hat{c}(p^*) \leq \left( 1 + \frac{3}{2\Gamma} \right) c(p^*). \]

Due to the construction of $\hat{G}$ and how $\hat{c}(p') \leq \hat{c}(p^*)$ we deduce that
\[ (c')^* \leq \hat{c}(p^*), \]
since
\[ (c')^* \leq c(p') \leq \hat{c}(p') \leq \hat{c}(p^*). \]
Therefore we have
\[ (c')^* \leq \left( 1 + \frac{3}{2\Gamma} \right) c^*.\]
\end{proof}

%******************************
%
%******************************
  The proof of the main result in Theorem~\ref{thrm:approximation_factor} directly follows from Lemma~\ref{thrm:2_factor_approximation} and Lemma~\ref{thrm:gamma_factor_approximation}.  

\begin{figure}
  \centering
  \includegraphics[width=0.9\linewidth]{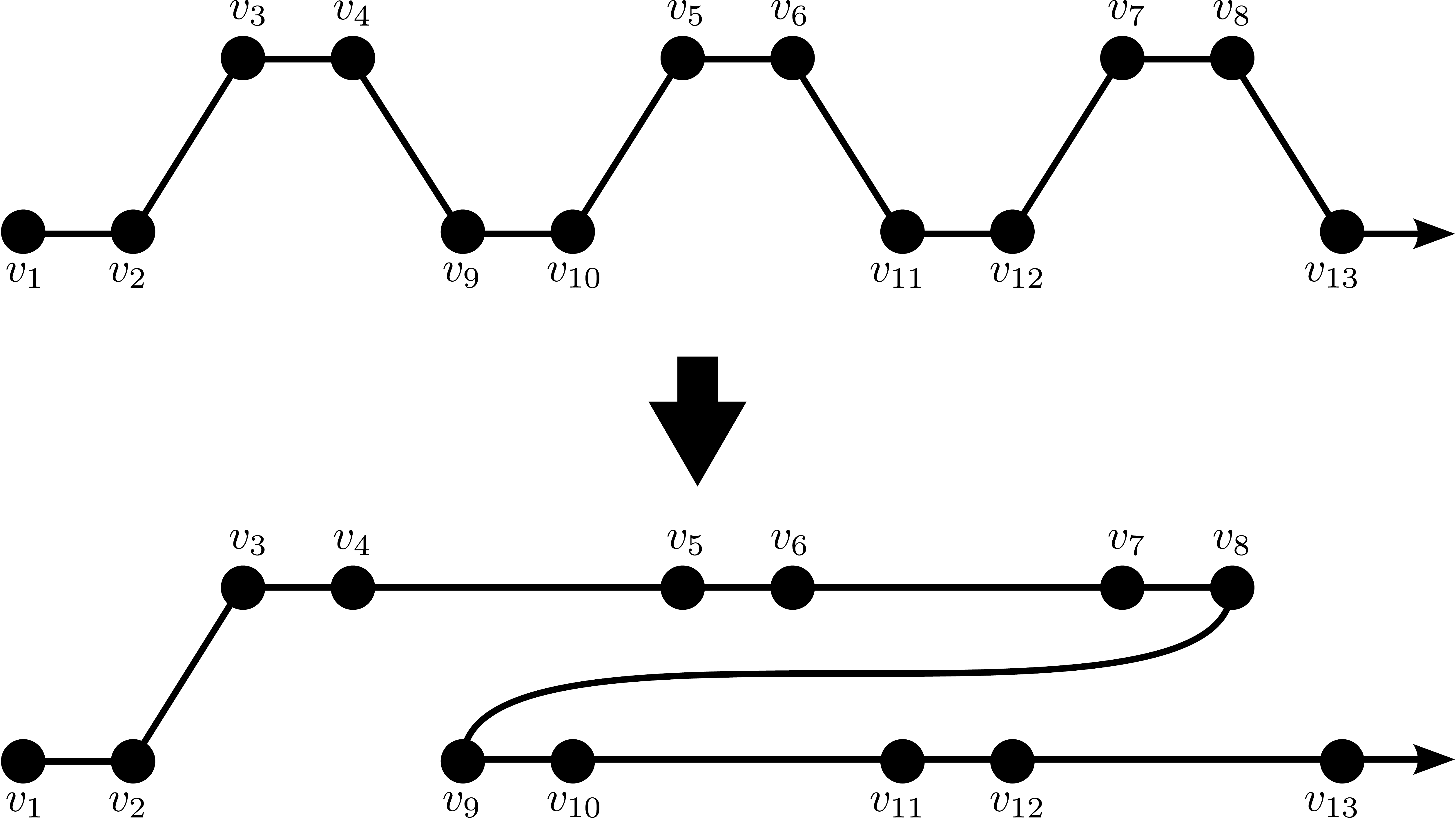}
  \caption{
    Path deformation example.
  }
  \label{fig:deform03}
\end{figure}

%******************************
%
%******************************
\begin{remark}[Tightness of Bound]
We have not fully characterized the tightness of the bounds from Lemma~\ref{thrm:2_factor_approximation} and Lemma~\ref{thrm:gamma_factor_approximation}, but we we have a lower bound, which is illustrated in Figure~\ref{fig:lower_bound}. In this example, the graph is scale-able (we can add vertices three at a time). The clustered and non-clustered solutions for this example have a cost relation of
\[ \lim_{|V| \rightarrow \infty} (c')^* = \left( 1 + \frac{2}{2\Gamma + 1} \right) c^*. \]
We also provide the graph in Figure~\ref{fig:bound} to show how the gap changes as $\Gamma$ varies.

The bound for this graph is obtained as follows. Let $n = \frac{|V|}{3} - 1$ for $\frac{|V|}{3} \in \ZZ_{>0}$. Then the optimal non-clustered solution cost is (recall that $\Gamma = \frac{\alpha}{\beta}$):
\begin{align*}
  p^* &= \tuple{v_1, v_2, v_3, v_4, v_6, v_5, \ldots} \\
  c^* &= \alpha + \beta + \alpha + n\left(2\alpha + \beta\right) \\
      &= \left( n+1 \right)\left( 2\alpha + \beta \right) \\
      &= \left( n+1 \right)\left( 2 + \frac{1}{\Gamma} \right) \alpha
\end{align*}

The optimal clustered solution is:
\begin{align*}
  (p')*   &= \tuple{v_1, v_2, v_3, v_5, v_6, \ldots, v_4, v_7, \ldots} \\
  (c')^*  &= \alpha + \beta + 2n\beta + \alpha + n(2\alpha + \beta) \\
          &= 2\alpha + 3\beta - 2\beta + n(2\alpha + 3\beta) \\
          &= (n+1)(2\alpha + 3\beta) - 2\beta \\
          &= \left( (n+1)(2 + \frac{3}{\Gamma}) - \frac{2}{\Gamma} \right) \alpha \\
\end{align*}

As the instance grows ($|V| \rightarrow \infty$, which implies $n \rightarrow \infty$), we have the following: 
\begin{align*}
  \lim_{n \rightarrow \infty} (c')^*
      &=  \frac{(n+1)\left(2 + \frac{3}{\Gamma} \right) - \frac{2}{\Gamma}}
               {\left( n+1 \right)\left(2 + \frac{1}{\Gamma} \right)} c^*  \\
      &=  \left( \frac{2 + \frac{3}{\Gamma}}{2 + \frac{1}{\Gamma}} \right) c^*  \\
      &=  \left( 1 + \frac{2}{2\Gamma + 1} \right) c^*  \\
\end{align*}
\end{remark}

\begin{figure}
  \centering
  \includegraphics[width=0.9\linewidth]{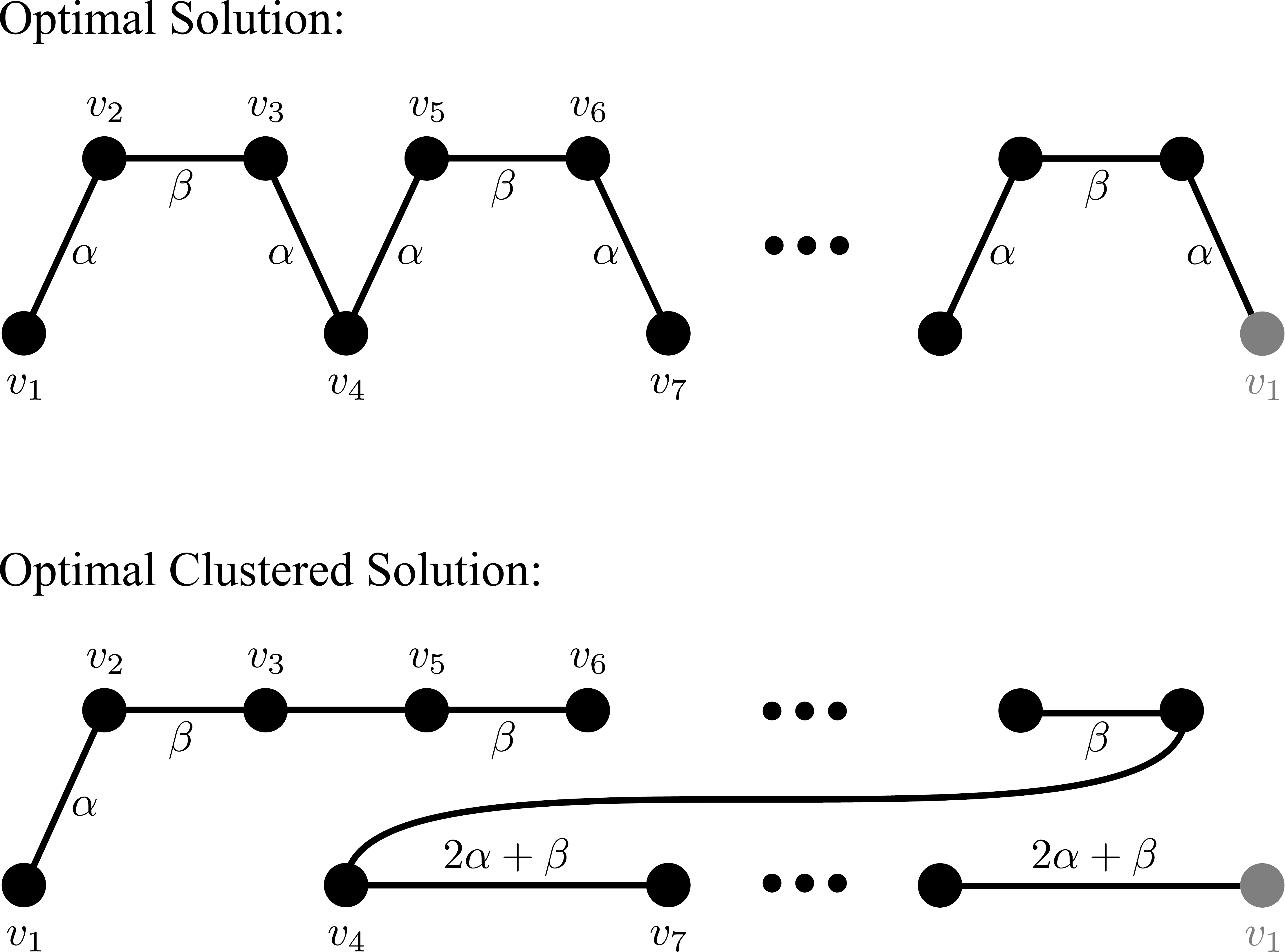}
  \caption{
    Metric example instance ($\alpha > \beta$). Vertices in the top row ($v_2,
    v_3, v_5, v_6, \ldots$) are in the cluster $V_i$ and the vertices in the
    bottom row are not. Edge weights connecting vertices within $V_i$ are
    $\beta$, edge weights connecting vertices not in $V_i$ are $2\alpha +
    \beta$, edge weights connecting vertices not in $V_i$ to vertices in $V_i$
    are $\alpha + \beta$ unless shown differently in diagram. Optimal solutions
    were verified with our solvers.
  }
  \label{fig:lower_bound}
\end{figure}

\begin{figure}
  \centering
  \includegraphics[width=0.8\linewidth]{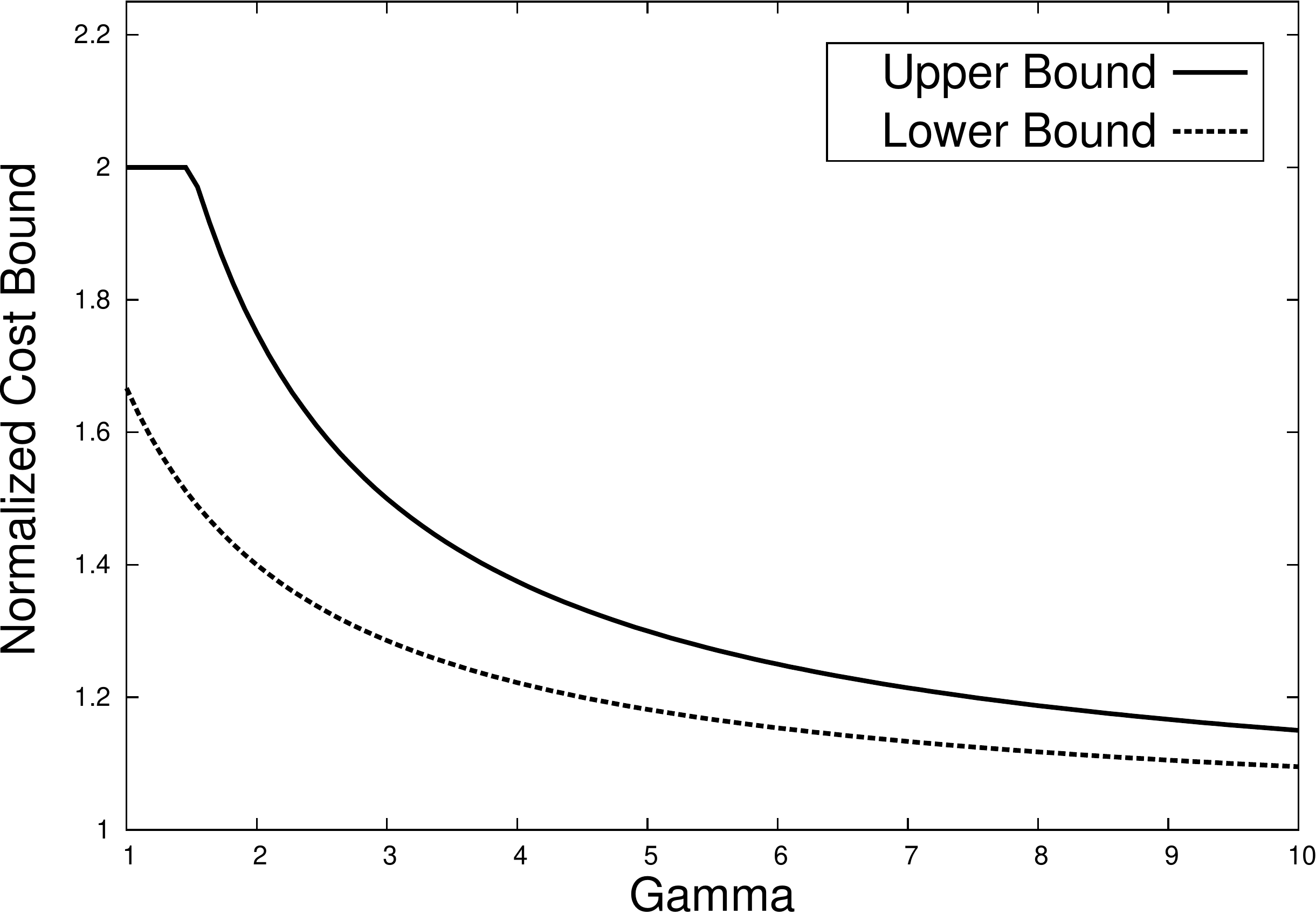}
  \caption{
    A plot showing the tightness of the approximation's upper bound. The gap
    between the two curves shows where the tightest upper bound can lie.
  }
  \label{fig:bound}
\end{figure}

%******************************
%
%******************************
\subsection{Finding \gammaCluster{s}}
Before we describe our method for finding optimal \gammaClustering(s), we describe a few properties of \gammaClustering{s}.

%***************
%
%***************
\subsubsection{Overlap}
A special property of \gammaClustering is that when $\Gamma > 1$, there are no overlapping clusters. Specifically there are no two clusters $V_i$ and $V_j$ with $\Gamma_i > 1$ and $\Gamma_j > 1$ that have a non-zero intersection unless one cluster is a subset of the other. 

\begin{lemma}
  \label{thrm:overlap}
  Given a graph $G$, clusters $V_i$ and $V_j$ with separations $\Gamma_i
  > 1$ and $\Gamma_j > 1$ (defined in Definition~\ref{def:gamma_cluster}), then
  $V_i \cap V_j$ is non-empty if and only if $V_i \subseteq V_j$ or $V_j
  \subseteq V_i$.
\end{lemma}

\begin{proof}
  We prove the above result by contradiction. Before we begin, recall that edges
  cut by the cluster $V_i$ have edge weights $\geq \alpha_i$ and edges within
  the cluster have edge weights $\leq \beta_i$. Let us assume that $V_i$ and
  $V_j$ overlap and are not nested. Without loss of generality let $\beta_j \leq
  \beta_i$. Then there exists an edge \tuple{v_a, v_b} with weight $w(v_a, v_b)
  \leq \beta_j$ for $v_a \in V_i \cap V_j$ and $v_b \in V_j \setminus (V_i \cap
  V_j)$. Since this edge does not exist in the cluster $V_i$ but it
  is cut by $V_i$, then it must be the case that $w(v_a, v_b) \geq \alpha_i$.
  However, this edge does exist in the cluster $V_j$ and so $\alpha_i \leq
  \beta_j$, since $w(v_a, v_b) \leq \beta_j$. This result highlights the
  contradiction: $\Gamma_i = \frac{\alpha_i}{\beta_i} \leq
  \frac{\beta_j}{\beta_i} \leq \frac{\beta_i}{\beta_i} = 1$.
\end{proof}

%***************
%
%***************
\subsubsection{Uniqueness}
The non-overlapping property for \gammaClustering{s} with $\Gamma > 1$ implies that there exists a unique maximal clustering set $C^*$ (more clusters equals more reductions in the search space size). This result follows from the simple property that if a cluster $V_i$ exists and does not exist in our clustering $C^*$, then we must be able to add it to $C^*$ to get additional search space reductions.

\begin{proposition}
  \label{thrm:superset}
  Given graph $G$ and a parameter $\Gamma > 1$, the problem of finding a
  \gammaClustering $C^*$ that maximizes the search space reduction has a unique
  solution $C^*$. Furthermore $C^*$ contains all clusters $V_i$ with separation
  $\Gamma_i \geq \Gamma$.
\end{proposition}

\begin{proof}
  We use contradiction to prove that $C^*$ is unique. Suppose there exists two
  different \gammaClustering{s} $C_1$ and $C_2$ that maximize the search space
  reductions for a given $\Gamma > 1$. Then, there is a cluster $V_1$ that is in
  $C_1$ and not in $C_2$ (or vice versa). This implies that either $V_1$ is
  somehow incompatible with $C_2$, which we know is not the case due to
  Lemma~\ref{thrm:overlap} (i.e., clusters do not overlap unless one is a subset
  of the other) or this cluster can be added to $C_2$. This is a contradiction,
  which proves the first result. The second result directly follows since adding
  a cluster can only further reduce the search space size of the problem.
\end{proof}

\begin{corollary}
  \label{thrm:gamma_gt_gamma}
  Given a graph $G$ and two clustering parameters $\Gamma_i > \Gamma_j > 1$,
  the optimal clustering $C^*_j$ for $\Gamma_j$ is a superset of any
  clustering for $\Gamma_i$.
\end{corollary}

\begin{proof}
  We prove this result by contradiction. Suppose we have a clustering $C_i$ for
  $\Gamma_i$ and the optimal clustering $C_j^*$ for $\Gamma_j$, for which there
  exists a cluster $V_1$ in $C_i$ that is not in $C_j^*$. Then by the definition
  of \gammaCluster{s} $V_1$ must satisfy $\Gamma_1 \geq \Gamma_i$ and since
  $\Gamma_i > \Gamma_j$ then $\Gamma_1 \geq \Gamma_j$, which makes $V_i$ a
  cluster that should be added to $C_j^*$. This contradicts
  Proposition~\ref{thrm:superset} and thus proves the result.
\end{proof}

\begin{corollary}
  There exists a minimum $\Gamma^* > 1$ and its optimal \gammaClustering $C^*$
  that is a superset of all other \gammaCluster{s} for $\Gamma > 1$.
\end{corollary}

\begin{proof}
  This result follows directly from Corollary~\ref{thrm:gamma_gt_gamma}, when we
  consider $\Gamma^*$ to be the smallest ratio of edge weights in the graph
  $G$ (find the two edges that give us the smallest ratio). Then for every
  cluster with a separation parameters $\Gamma_i > 1$, it must also be greater
  than or equal to $\Gamma^*$ since $\Gamma_i$ itself is a ratio of existing edge
  weights in the graph $G$. Thus by Corollary~\ref{thrm:gamma_gt_gamma}, $C^*$
  must be superset of any other \gammaClustering for some $\Gamma > 1$.
\end{proof}

%***************
%
%***************
\subsubsection{An \mst Approach For Finding \gammaCluster{s}}
Given an input $\Gamma > 1$, Algorithm~\ref{alg:clustering} computes the optimal \gammaClustering{s}, i.e., the \gammaClustering with maximum search space reduction. Informally the algorithm deletes edges in the graph from largest to smallest (line 6-7) to look for \gammaCluster{s}. It uses a minimum spanning tree (\mst) to keep track of when the graph becomes disconnected, and when it does the disconnected components are tested to see if they qualify as \gammaCluster{s} (line 9-11). Regardless, any non-trivial sized disconnected component (\gammaCluster or not) is added back to the queue (line 13-14), so that it can be broken and tested again to find of all nested \gammaCluster{s}.

\begin{algorithm}
  \caption{\textsc{\gammaClustering}$(G,\Gamma)$}
  \label{alg:clustering}
  \texttt{\bf assert($\Gamma > 1$)} \\
  $C \leftarrow \{\}$ \\
  $M \leftarrow \{\mst(G)\}$ \\
  \While{$|M| > 0$}{
    $m \leftarrow M.pop()$ \\
    $\alpha \leftarrow $ largest edge cost in $m$ \\
    $M' \leftarrow $ disconnected trees after removing edge(s) of cost $\alpha$ from $m$ \\
    \For{$m' \in M'$}{
      $G' \leftarrow $ graph induced by $V[m']$ \\
      $\beta \leftarrow $ max edge cost of $G'$ \\
      \If{$G'$ is a clique and $\Gamma' \equiv \alpha/\beta \geq \Gamma$}{
        $C \leftarrow C \cup \{V[m']\}$ \\
      }
      \If{$|m'| > 1$} {
        $M \leftarrow M \cup m'$ \\
      }
    }  
  }
  \Return $C$
\end{algorithm}

\begin{theorem}
  \label{thrm:clustering}
  Given $G$ and a $\Gamma > 1$, Algorithm~\ref{alg:clustering} finds the optimal
  \gammaClustering $C^*$ in $O(|V|^3)$ time.
\end{theorem}

\begin{proof}
  The proof is two fold, first we show that Algorithm~\ref{alg:clustering} runs
  in polynomial time and second we show that it finds the optimal
  \gammaClustering. For this proof let $n = |V|$.

  Minimum spanning trees can be found $O(n^2)$ time~\cite{prim1957shortest}
  (line 3). The rest of the algorithm modifies the \mst from line 3, which
  originally has $n$ edges and so the while loop for the rest of the algorithm
  can only run at most $n$ times (we can't remove more than $n$ edges). Finding
  the largest edge(s) and removing them in line 6 and 7 takes $O(n)$ time.
  Creating the induced subgraph and finding the maximum edge cost in lines 9 and
  10 takes $O(n^2)$ time. Testing if the subgraph is a clique (line 11) takes
  $O(n^2)$ time. Thus lines 1 to 3 run in $O(n^2)$ time and lines 4 to 14 run in
  $O(n\cdot n^2)$, which means the entire algorithm runs in $O(n^3)$ time or
  $O(|V|^3)$.
  
  Next we show that Algorithm~\ref{alg:clustering} finds the optimal
  \gammaClustering. To do so, we will show that $V[m']$ does indeed represent a
  \gammaClustering with separation $\Gamma' \equiv \frac{\alpha}{\beta}$ (line
  11) as defined by Definition~\ref{def:gamma_cluster} and then we  finish
  by showing that the algorithm does not omit any candidate clusters, thus
  proving the theorem by leveraging Proposition~\ref{thrm:superset}.

  Let us start by understanding how to find \gammaCluster{s} and how we use
  \mst{s} in the algorithm. A \gammaCluster with separation $\Gamma_i > 1$ is a
  subgraph of $G$ that is connected with intra-edge weights less than the
  inter-edge weights connected to the rest of the graph. Thus one method of
  searching for \gammaCluster{s} is to delete all edges in the graph of weight
  $\geq \alpha$ and if there is a disconnected subgraph in $G$ then and only
  then is it a possible \gammaCluster. We can use \mst{s} to more efficiently
  keep track of these deleted edges. By definition, an \mst is tree that connects
  vertices within the graph with minimum edge weight. If we remove edges
  of size $\geq \alpha$ in the graph to search for disconnected sub-graphs, then
  the graph is disconnected if and only if the \mst is disconnected. This
  follows by considering the cut needed to disconnect these two sub-graphs (the
  minimum weight edge cut, shares the same weight cut by the \mst). Thus we can
  instead search for \gammaCluster{s} by disconnecting the \mst. Furthermore if
  we disconnected the \mst by incrementally removing the largest edge(s) of
  weight $\alpha$ from the tree then we know that induced sub-graphs of the
  newly disconnected trees have at least one inter-set edge of size $\alpha$.
  Line 9 in the Algorithm creates the induced subgraph $G'$ of a disconnected
  tree ($m' \in M'$), line 10 measures its $\beta$, and if $G'$ is a clique and
  meets the separation criterion ($\Gamma' \geq \Gamma$) then $V[m']$ is indeed
  a \gammaCluster (by definition), thus it is added to the clustering $C$ in
  line 12. 
  
  Next we show that the Algorithm did not omit any candidate clusters. We have
  already argued that only disconnected \mst trees need to be considered for
  \gammaCluster{s}, thus what is left to show is that the algorithm tests every
  possible value of $\alpha$ to disconnect the \mst tree.  This is true since it
  considers every edge in the original \mst. This is because every disconnected
  tree of size two or more is added back to $M$ in line 14 until every edge
  originally in the \mst is removed in line 7. 

  Therefore this algorithm finds all of the candidate \gammaCluster{s} and
  Proposition~\ref{thrm:superset} tells us that it is the unique optimal
  \gammaClustering for the given $\Gamma > 1$.
\end{proof}

%******************************
%
%******************************
\subsection{Search Space Reduction}
The last remaining question is to determine how much the clustering approach reduces the search space. In general, this is difficult to answer since it depends on the particular constraints of the path planning problem.  However, to get an understanding of the search space reduction, consider the example of \tsp with a non-nested clustering (nesting would result in further search space reductions). Let $r$ be the ratio of the non-clustered search space size $N_0$, to the clustered search space size $N_1$, for a graph with vertices $V$ and a clustering $C = \{ V_1, V_2, \dots, V_m \}$. Then, the ratio (derived from counting the number of solutions) is as follows:
\[ 
r \equiv \frac{N_0}{N_1} = \frac{|V|!}{m! \prod_{i}^{m} |V_i|!} 
\]
To further simplify the ratio consider the case where all clusters are equally sized (for all $V_i, V_j \in C$, $|V_i| = |V_j|$):

\begin{proposition}[Search Space Reduction]
  Given a graph $G$ and a clustering $C = \{V_1, V_2, \ldots, V_m\}$ such that
  $|V_i| = |V_j|$ for all $i,j \in [1,m]$ then the ratio $r$ of the search
  space size for the original \tsp problem to the cluster \tsp problem is
  \footnote{The big-$\Omega$ notation states that for large enough
  $|V|$ the ratio $r$ is at least $k \dot (m!)^{x-1}$ for some constant $k$.}
  \[ r = \Omega\left( (m!)^{x-1} \right), \]
  where $x = |V|/m$.
\end{proposition}

\begin{proof}
  The number of solutions for the (directed) \tsp problem is $N_0 = |V|!$ and
  the number of solutions for the \nestedCtsp problem is: $N_1 = m! \left(
  \frac{|V|}{m} ! \right)^m$ (these results come from counting the number of
  possible solutions). First let us bound $|V|! = (mx)!$:
  \begingroup
  \allowdisplaybreaks
  \begin{align*}
    |V|!  &=    \prod_{i=m}^{1} \prod_{j=0}^{x-1} \left( ix - j \right) \\
          &=    \prod_{i=m}^{1} \prod_{j=0}^{x-1} (i) \left( x - \frac{j}{i} \right) \\
          &\geq \prod_{i=m}^{1} \prod_{j=0}^{x-1} (i) \left( x - j \right) \\
          &=    \left( \prod_{i=m}^{1} i^x \right) \left( \prod_{j=0}^{x-1} ( x - j ) \right)^m \\
          &=    (m!)^x(x!)^m
  \end{align*}
  \endgroup
  
  We now use the fact that $|V|! \geq (m!)^x(x!)^m$ to prove the main result: 
  \begin{align*}
    r & \equiv  \frac{N_0}{N_1}  =       \frac{|V|!}{m!\left(x!\right)^m}  \geq    \frac{\left(m!\right)^x \left(x!\right)^m}{m! \left(x!\right)^m}  =       (m!)^{x-1} \\
  \end{align*}
  Thus $r = \Omega\left( (m!)^{x-1} \right)$.
\end{proof}

To get an idea of the magnitude of $r$, consider an instance of size $|V|=100$ divided into four equal clusters. The clustered problem has a feasible solution space of size $N_1 \approx  1.49 \times 10^{-56} N_0$, where $N_0$ is the feasible solution space size of the non-clustered problem. However, $N_1$ is still extremely large at about $1.39 \times 10^{102}$.

%******************************************************************************************
% 
%******************************************************************************************
\section{\gammaClustering Experiments}
\label{sec:experiments}
In this section we present experimental results that demonstrate the effectiveness of \gammaClustering for solving discrete path planning problems. We focus on metric \tsp instances drawn from the established \tsp library \tsplib~\cite{reinelt1991tsplib}, which have a variety of \tsp problem types (the first portion of the instance name indicates the type and the number  indicates the size). The tests were conducted with $\Gamma = 1.000001$, for which Theorem~\ref{thrm:approximation_factor} implies that the solution to the clustered problem gives a $\min\left(2, 1 + \frac{3}{2\Gamma}\right)$-factor approximation to the \tsp instance. However, we will see that the observed gap in performance is considerably smaller.

To test the effectiveness of the clustering method, we perform \gammaClustering on each \tsp instance and recorded both the runtime and the number of clusters found. This gives us an idea of whether or not instances from \tsplib have a structure that can be exploited by \gammaClustering. Then, we use standard integer programming formulations for both the original \tsp instance and the general clustered version of the instance (\nestedCtsp). We solve each instance three times using the solver Gurobi~\cite{optimization2012gurobi} and recorded the average solver time and solution quality. All instances were given a time budget of 900 seconds, after which they were terminated and the best solution found in that run was outputted. The \tsp instances reported in this paper are the 37 instances that were solved to within 50\% of optimal and have \gammaCluster{s}. The remaining instances are not reported as they would have required more than 900 seconds to provide a meaningful comparison.

Additionally we demonstrate \gammaClustering on an office environment to gain insight into the structure of \gammaCluster{s}.

%******************************
%
%******************************
\subsection{Integer Programming Formulation}
The clustering algorithm was implemented in Python and run on an Intel Core i7-6700, 3.40GHz with 16GB of RAM. The integer programming (IP) expression of the \tsp problem and clustered problem is solved on the same computer with Gurobi, also accessed through python. The results of both of these approaches is found in Table~\ref{table:results} and summarized in Figure~\ref{fig:results}.

The following is the IP expression used for the clustered and non-clustered path planning problems, where each variable $e_{a,b} \in \{0,1\}$:
\setcounter{equation}{0}
\begin{align}
  \text{minimize} \nonumber \\
    & \sum_{a=1}^{|V|} \displaystyle\sum_{b=1}^{|V|} e_{a,b}\ w(v_a, v_b) \\
  \text{subject to} \nonumber \\
    & \sum_{b=1}^{|V|} e_{a,b} = 1, 
      \quad\text{ for each } a \in \{ 1,2,\ldots, |V| \}
      \label{eq:v_outgoing} \\
    & \displaystyle\sum_{a=1}^{|V|} e_{a,b} = 1,
      \quad\text{ for each } b \in \{ 1,2,\dots, |V| \}
      \label{eq:v_incoming} \\
    & \hspace*{-4ex}
      \displaystyle\sum_{\forall v_a \in V_i, v_b \not\in V_i}
        \hspace*{-4ex}
        e_{a,b} = 1,
      \quad\text{ for each } i \in \{ 1,2,\ldots, m \}
      \label{eq:set_outgoing} \\
    & \hspace*{-4ex}
      \displaystyle\sum_{\forall v_a \not\in V_i, v_b \in V_i}
        \hspace*{-4ex}
        e_{a,b} = 1,
      \quad\text{ for each } i \in \{ 1,2,\ldots, m \}
      \label{eq:set_incoming} \\
    & \hspace*{-1ex}
      \displaystyle\sum_{e_{a,b} \in E'}
        \hspace*{-1ex}
        e_{a,b} \leq |E'|-1,
      \quad\text{ for each subtour } E'
      \label{eq:sub_tour_elim}
\end{align}
The formulation was adapted from a \tsp IP formulation found in~\cite{gouveia1999asymmetric}, where the Boolean variables $e_{a,b}$ represent the inclusion/exclusion of the edge $\tuple{v_a, v_b}$ from the solution.

Constraints~\ref{eq:v_outgoing} and~\ref{eq:v_incoming} restricts the incoming and outgoing degree of each vertex to be exactly one (the vertex is visited exactly once). Similarly constraints ~\ref{eq:set_outgoing} and ~\ref{eq:set_incoming} restrict the incoming and outgoing degree of each cluster to be exactly one (these constraints are only present in the clustered version of the problem). Constraint~\ref{eq:sub_tour_elim} is the subtour elimination constraint, which is lazily added to the formulation as conflicts occur due to the exponential number of these constraints. For each instance, we seed the solver with a random initial feasible solution.

%******************************
%
%******************************
\subsection{Results}
Figure~\ref{fig:time} shows the ratio of time spent finding clusters with respect to total solver time. In all instances this time is less than 6\% and most is less than 1\%. Additionally as the total solver time grows, the ratio gets smaller (Time02 is for instances that use the full 900 seconds). This approach is able to find \gammaCluster{s} on 63 out of 70 \tsp instances. In total it found 3700 non-trivial \gammaCluster{s} (i.e., clusters with $|V_i| \geq 2$), which is promising since the \tsplib library contains a variety of different \tsp applications.

Figure~\ref{fig:results} and Table~\ref{table:results}, shows that for instances that do not time out the solution path costs found by the clustering approach are close to optimal (Cost01 is close to 1 in the figure and the instances from burma14 to pr107 are almost all within 1\% error as shown in the table). Furthermore when the solver starts to time out (exceeds its 900 second time budget) the solution quality of the \gammaClustering approach starts to surpass the solution quality of the non-clustered approach (as shown by Cost02 in the figure). We attribute this trend to the fact that the clustered approach needs less time to search its feasible solution space and thus is able to find better quality solutions faster than its counterpart. In instances gr202, kroB200, pr107, tsp225, and gr229 the clustering approach does very well compared to the non-clustering approach, enabling the solver to find solutions within 1\% of optimal for the first three instances and solutions within 6\% of optimal for the latter two instances, while the non-clustered approach exceeds 8\% of optimal in the first three instances and 28\% of optimal for the latter two instances.

On average the \gammaClustering approach is more efficient than the non-clustered approach, which is highlighted in Figure~\ref{fig:results} and Table~\ref{table:results}. For the results that do not time out (Time01 in the figure) we often save more than 50\% of the computational time, while maintaining a near optimal solution quality (Cost01 in the figure). For instances that require most or all of the 900 seconds (harder instances) the time savings can be quite large. This is particularly clear from the table when we compare the easy instances burma14 up to st70 that have an average time savings of around 60\% to the harder instances kroA100, gr96, bier127, and ch130, which all have a time savings of more than 95\%. For the instances that both time out (Time02 in the figure), the time savings is not present since both solvers use the full 900 seconds.

From these results we can see that when the \gammaClustering approach does not time out, we usually save time and when the approach does time out we often find better quality solutions as compared to the non-clustered approach.

\begin{remark}[Discussion on using \gammaClustering for \tsp]
  It is worth emphasizing that we are not necessarily recommending solving \tsp
  instances in this manner. We are simply using \tsp as an illustrative example
  to show how \gammaClustering can be used to reduce computation time in a given
  solver. Many discrete path planning problems are solved with IP solvers and as
  such we hope our results help provide some insight as to how \gammaClustering
  would work on other path planning problems. In general unless the solver
  approach (IP or not) takes advantage of the clustering, there is no
  guarantee that a computation saving will be achieved.
\end{remark}

%******************************
%
%******************************
\subsection{Clustering Real-World Environments}
As shown in Figure~\ref{fig:example} we have also performed \gammaClustering on real-world environments.  The figure shows the floor-plan for a portion of one floor of the Engineering~5 building at the University of Waterloo.  Red dots denote the locations of desks within the environment. We encoded the environment with a graph, where there is a vertex for each red dot and edge weights between vertices are given by the shortest axis-aligned obstacle free path between the locations (obstacle-free Manhattan distances). The figure shows the results of clustering for $\Gamma = 1.000001$. We see that locations that are close together are formed into clusters unless there are other vertices in close proximity.

A path planning problem on this environment could be robotic mail delivery, where a subset of locations must be visited each day. The clusters could then be visited together (or visited using the same robot).

\begin{figure}
  \centering
  \includegraphics[height=120pt]{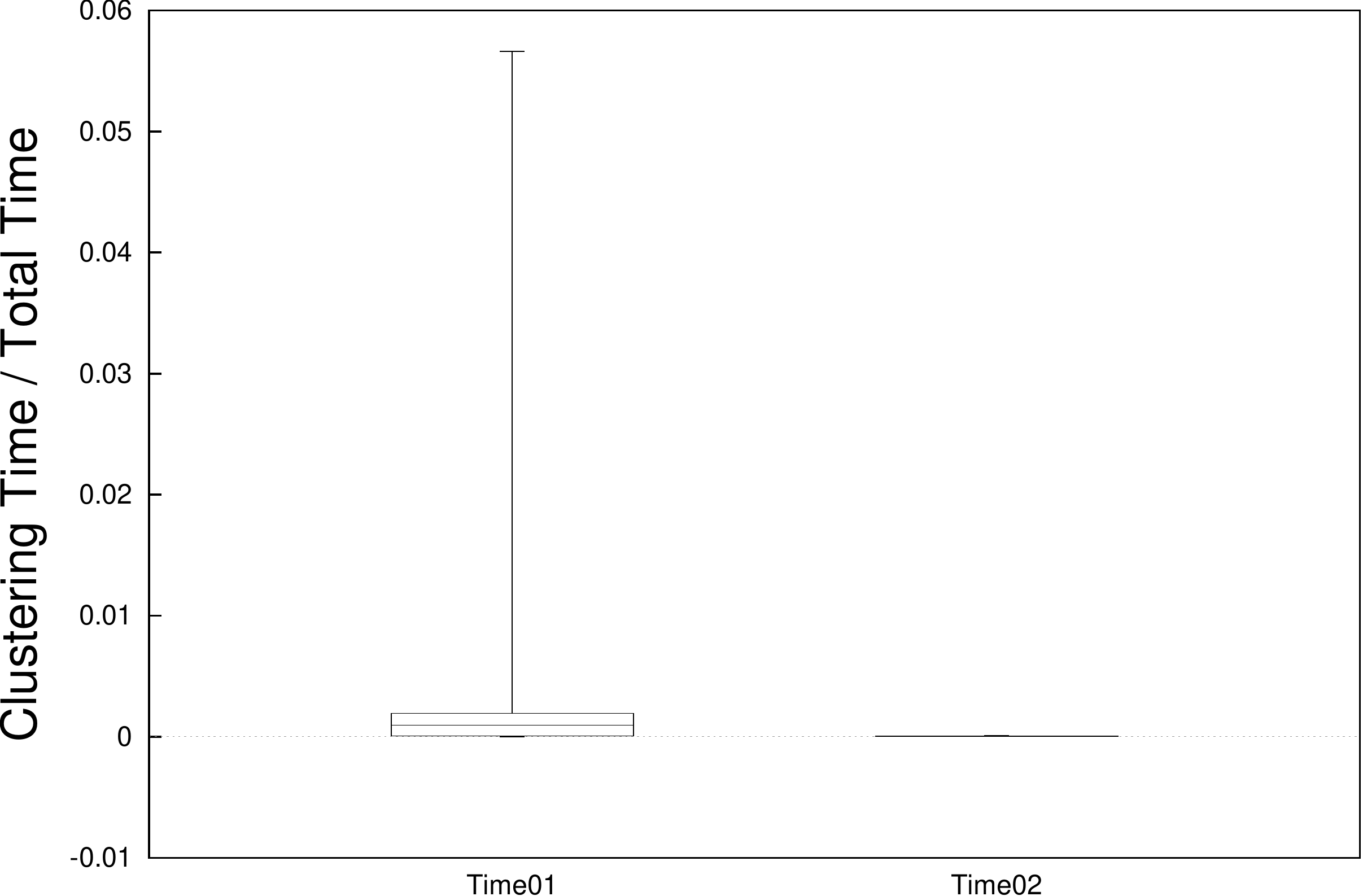}
  \caption{
    Box plot of the clustering time ratio with respect to the \gammaClustering
    approach. The data is categorized by instances that did not time out
    (Time01) and instances that did time out (Time02).
  }
  \label{fig:time}
\end{figure}

\begin{figure}
  \centering
  \includegraphics[height=120pt]{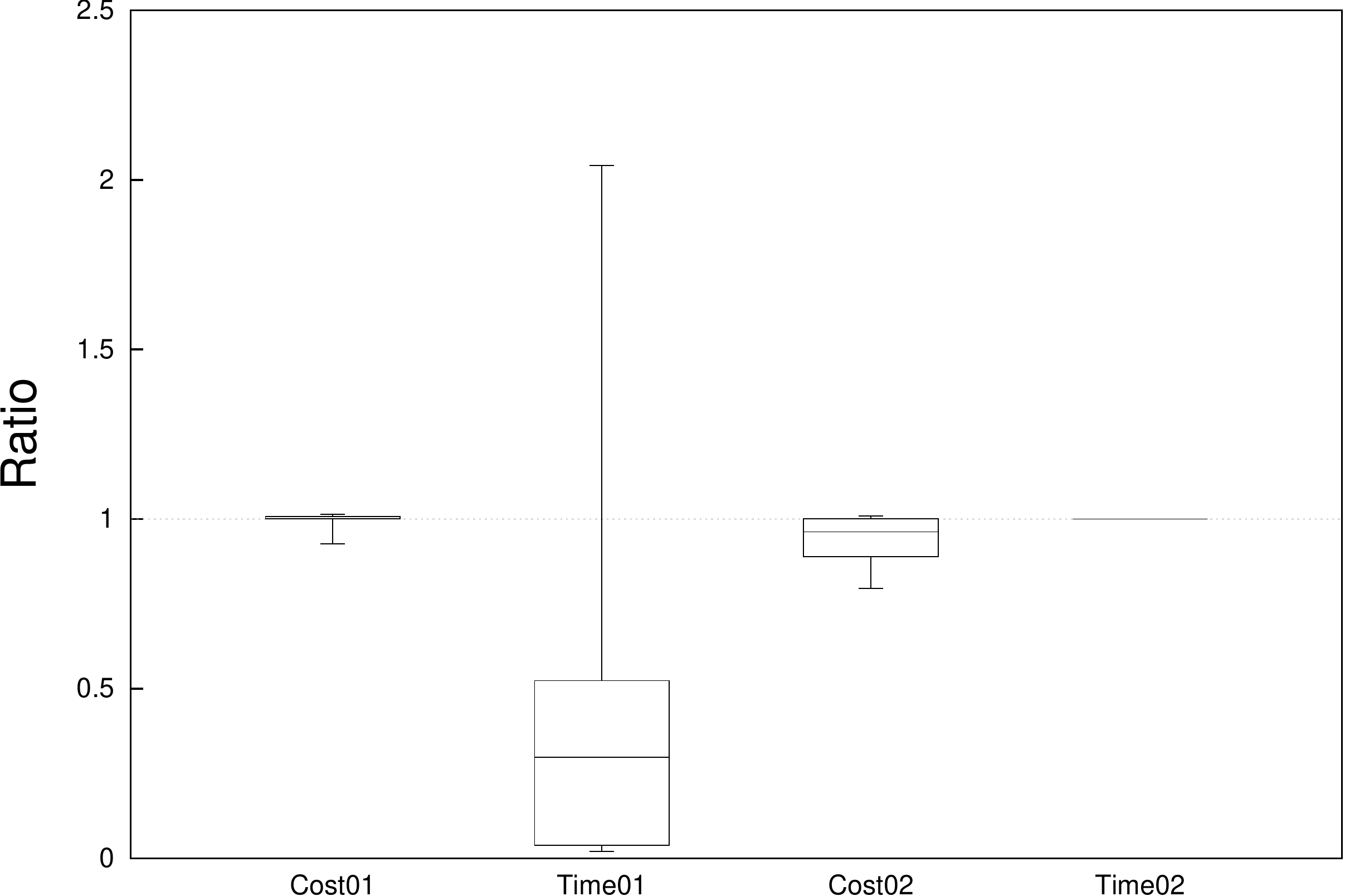}
  \caption{
    Box plot of the solver cost and time ratios for the \gammaClustering
    approach with respect to the \tsp approach. The data is categorized by
    instances that did not time out (Cost01 and Time01) and instances that did
    time out (Cost02 and Time02).
  }
  \label{fig:results}
\end{figure}

\begin{table}
	\centering
	\scriptsize
	\begin{tabular}{@{}lc rr rr}
	                    & 
	                    & \multicolumn{2}{c}{\tsp} 
	                    & \multicolumn{2}{c}{\nestedCtsp} \\
	                    & $|C|$
                      & \% Error & Time ($t$)
                      & \% Error & Time ($t$) \\
    \midrule
    burma14 & 5 & 0.00 & $<$1 & 0.39 & {$<$1} \\ 
    ulysses16 & 6 & 0.00 & $<$1 & 0.73 & {$<$1} \\ 
    berlin52 & 17 & 0.00 & $<$1 & 0.06 & {$<$1} \\ 
    swiss42 & 16 & 0.00 & $<$1 & 0.94 & {$<$1} \\ 
    eil51 & 11 & 0.00 & 1 & 0.00 & {\bf $<$1} \\ 
    eil76 & 13 & 0.00 & 3 & 0.00 & {\bf 1} \\ 
    rat99 & 28 & 0.00 & 6 & 0.83 & 13 \\ 
    eil101 & 16 & 0.00 & 8 & 0.00 & {\bf 4} \\ 
    ulysses22 & 10 & 0.00 & 8 & 0.00 & {\bf $<$1} \\ 
    pr76 & 27 & 0.00 & 40 & 1.40 & {\bf 11} \\ 
    st70 & 23 & 0.00 & 45 & 0.44 & {\bf 13} \\ 
    lin105 & 42 & 0.00 & 225 & 0.00 & {\bf 8} \\ 
    kroE100 & 43 & 0.00 & 414 & 0.39 & {\bf 20} \\ 
    kroC100 & 46 & 0.00 & 445 & 0.00 & {\bf 12} \\ 
    kroA100 & 44 & 0.00 & 602 & 1.11 & {\bf 23} \\ 
    gr96 & 32 & 0.00 & 845 & 0.05 & {\bf 38} \\ 
    bier127 & 37 & 0.00 & 900 & 0.23 & {\bf 24} \\ 
    gr137 & 44 & 0.00 & 900 & 0.00 & {\bf 154} \\ 
    kroD100 & 42 & 0.00 & 900 & 0.57 & {\bf 471} \\ 
    kroB100 & 42 & 0.01 & 900 & 0.96 & 900 \\ 
    ch130 & 59 & 0.20 & 900 & 0.90 & {\bf 29} \\ 
    ch150 & 53 & 0.22 & 900 & 0.31 & 900 \\ 
    kroB150 & 65 & 0.35 & 900 & {0.35} & {\bf 395} \\ 
    kroA150 & 58 & 1.37 & 900 & {\bf 0.15} & {\bf 492} \\ 
    rat195 & 57 & 1.75 & 900 & {\bf 0.89} & 900 \\ 
    kroA200 & 82 & 5.59 & 900 & {\bf 1.38} & 900 \\ 
    pr124 & 18 & 5.78 & 900 & {\bf 4.24} & 900 \\ 
    gr202 & 74 & 8.56 & 900 & {\bf 0.64} & {\bf 703} \\ 
    pr136 & 48 & 9.78 & 900 & {\bf 4.88} & 900 \\ 
    kroB200 & 80 & 10.16 & 900 & {\bf 0.67} & 900 \\ 
    pr107 & 6 & 13.01 & 900 & {\bf 0.40} & 900 \\ 
    a280 & 11 & 20.11 & 900 & 20.46 & 900 \\ 
    pr144 & 40 & 24.40 & 900 & {\bf 20.03} & 900 \\ 
    tsp225 & 66 & 28.50 & 900 & {\bf 5.76} & 900 \\ 
    gr229 & 73 & 28.63 & 900 & {\bf 2.24} & 900 \\ 
    pr152 & 44 & 40.64 & 900 & 41.15 & 900 \\ 
    gil262 & 98 & 44.81 & 900 & {\bf 21.00} & 900 \\ 
    \bottomrule
	\end{tabular}
	\caption{
	  Experimental results for \tsplib instances: $|C|$ reports the number of
	  \gammaCluster{s}, \% error reports the average error from optimal and time
	  column reports the average solver time. The instances where the \nestedCtsp
	  approach outperforms the \tsp approach are highlighted in bold. Results are
	  sorted from least to most difficult for the non-clustered approach.
	}
	\label{table:results}
\end{table}

%******************************************************************************************
% 
%******************************************************************************************
\section{Conclusion}
\label{sec:conclusion}
In this paper we presented a new clustering approach called \gammaClustering. We have shown how it can be used to approximate the discrete path planning problems to within a constant factor of $\min\left(2, 1 + \frac{3}{2\Gamma}\right)$, more efficiently than solving the original problem. We verify these findings with a set of experiments that show on average a time savings and a solution quality that is closer to the optimal solution than it is to the bound. For future directions we will be investigating other path planning applications, which includes online path planning with dynamic environments.


\end{document}